\documentclass{article}

\PassOptionsToPackage{numbers}{natbib}
\usepackage[preprint]{neurips_2026}

\usepackage[utf8]{inputenc}
\usepackage[T1]{fontenc}
\usepackage{hyperref}
\usepackage{url}
\usepackage{booktabs}
\usepackage{amsfonts}
\usepackage{nicefrac}
\usepackage{microtype}
\usepackage{xcolor}
\usepackage{graphicx}
\usepackage{amsmath}
\usepackage{amssymb}
\usepackage{amsthm}
\usepackage{algorithm}
\usepackage{algorithmic}
\usepackage{mathtools}  

\newtheorem{proposition}{Proposition}

\title{Budget-Constrained Causal Bandits: Bridging Uplift Modeling and Sequential Decision-Making}

\author{%
  Abhirami Pillai\\
  \texttt{arp220005@utdallas.edu} \\
}

\begin{document}

\maketitle

\begin{abstract}
Treatment allocation under budget constraints is a central challenge in digital advertising: advertisers must decide which users to show ads to while spending a limited budget wisely. The standard approach follows a two-stage offline pipeline---first collect historical data to estimate heterogeneous treatment effects (HTE), then solve a constrained optimization to allocate the budget. This works well with abundant data, but fails in cold-start settings such as new campaigns, new markets, or new customer segments where little historical data exists. We propose Budget-Constrained Causal Bandits (BCCB), an online framework that learns which users respond to ads while simultaneously spending the budget, making treatment decisions one user at a time. BCCB unifies three components into a single sequential process: learning individual-level ad effectiveness, exploring users whose response is uncertain, and pacing the budget over time. We derive the per-arrival decision rule as the KKT condition of a Lagrangian relaxation of the budgeted causal-allocation objective, providing a principled foundation for the algorithm. We evaluate on the Criteo Uplift dataset, a large-scale benchmark derived from a real randomized controlled trial, using 20 random seeds with paired statistical tests. Our central empirical finding is a data-efficiency crossover identified by formal hypothesis test at $n = 7{,}500$ historical observations (paired one-sided $t$-test, $p = 0.043$): below this threshold, offline uplift pipelines either fail entirely or produce highly unreliable allocations, while BCCB operates effectively from the very first user. BCCB exhibits 2--4$\times$ lower run-to-run variance than offline methods across the training-size range tested, and outperforms all four online baselines---standard Thompson Sampling, budgeted Thompson Sampling, HTE Greedy, and Uplifting Bandits~\citep{hsieh2022}, a contextual causal bandit that also estimates per-user treatment effects---at every budget level tested with paired $t$-test $p < 0.001$. These results provide practitioners with a concrete decision rule for choosing between offline and online paradigms in budget-constrained treatment allocation.
\end{abstract}

\section{Introduction}

Digital advertising is a multi-billion dollar industry where advertisers must allocate limited budgets across large user populations. The fundamental challenge is that not all users respond equally to advertising---some users are strongly influenced by ads, others would have converted regardless, and some are entirely unaffected. This variation in individual-level response, known as heterogeneous treatment effects (HTE) \citep{athey2016, wager2018}, means that treating all users equally wastes budget on users who would not benefit from the ad.

A substantial body of work addresses this problem through uplift modeling: estimating the causal effect of showing an ad to each individual user, then targeting those with the highest predicted response \citep{gutierrez2017, diemert2018}. When combined with budget constraints, this becomes a resource allocation problem: select the set of users to treat that maximizes total uplift without exceeding the budget. The dominant approach follows a two-stage pipeline: first train an uplift model on historical randomized trial data, then solve a constrained optimization typically formulated as a Multi-Choice Knapsack Problem to determine which users to target \citep{zhao2019, albert2022, zhou2023}.

However, this two-stage pipeline suffers from well-documented limitations. The prediction and optimization stages are trained with misaligned objectives: the machine learning model minimizes statistical error globally, while the optimizer requires accuracy specifically at the budget-induced decision boundary. Errors from the prediction stage cascade into the optimization stage and are amplified by the deterministic solver. Furthermore, unconstrained predictive models often produce non-monotonic uplift curves that violate basic economic assumptions about consumer behavior.

Recent work has made significant progress in addressing these limitations through end-to-end Decision-Focused Learning (DFL), which unifies prediction and optimization into a single differentiable architecture. The End-to-End Cost-Effective Incentive Recommendation model (E3IR) integrates Integer Linear Programming directly into the neural network as a differentiable layer, enforcing monotonicity and smoothness through Lipschitz regularization \citep{du2024}. Bi-Level Decision-Focused Causal Learning (Bi-DFCL) resolves the counterfactual problem by using scarce randomized trial data to correct biased observational models through implicit differentiation \citep{bidfcl2025}. Direct Heterogeneous Causal Learning (DHCL) prevents error cascading by training models to directly predict a composite decision factor rather than raw uplift estimates \citep{dhcl2023}. Large-Scale Budget-Constrained Causal Forests (LBCF) integrate budget constraints directly into tree-splitting criteria for distributed computing environments \citep{lbcf2022}.

Despite this progress, all existing approaches---whether two-stage or end-to-end---share a fundamental assumption: the availability of sufficient historical data for offline model training. This assumption frequently breaks down in practice. New advertising campaigns have no historical data. Expansion into new markets or customer segments means historical data may not be representative. Seasonal shifts, product launches, and changing user behavior can render historical models outdated. In all these scenarios, the offline paradigm---regardless of its architectural sophistication---either fails entirely or produces unreliable allocations.

An alternative paradigm exists in the multi-armed bandit literature, where decisions are made sequentially and the algorithm learns from each observation in real time \citep{thompson1933, russo2018}. Budgeted bandits extend this framework to settings where each decision has a cost and total spending is constrained \citep{xia2015, wu2024}. However, existing budgeted bandit methods do not explicitly model heterogeneous treatment effects---they treat the problem as choosing between generic arms rather than estimating individual-level causal effects that vary with user characteristics.

In this paper, we bridge these paradigms. We propose Budget-Constrained Causal Bandits (BCCB), an online framework that combines HTE learning with Thompson-Sampling exploration and adaptive budget pacing. Unlike all offline methods---whether two-stage or end-to-end---BCCB does not require pre-collected historical data. It learns and allocates simultaneously, making treatment decisions one user at a time. Unlike standard bandit methods, BCCB explicitly models heterogeneous treatment effects and incorporates per-user cost variation into its decision rule.

Our main contributions are as follows:
\begin{enumerate}
    \item We propose BCCB, a unified online framework for budget-constrained treatment allocation that combines HTE estimation, Thompson Sampling exploration, and adaptive budget pacing into a single sequential decision process. We derive its per-arrival decision rule from a Lagrangian relaxation of the budget-constrained policy-optimization objective, providing a principled foundation for the heuristic combination of components.
    \item We provide an empirical characterization of the data-efficiency tradeoff between offline and online paradigms on the Criteo Uplift dataset \citep{diemert2018}. We identify a crossover point by formal hypothesis test at $n = 7{,}500$ historical observations (paired one-sided $t$-test, $p = 0.043$), persistent at all larger training sizes ($p < 0.001$ for $n \ge 10{,}000$).
    \item We demonstrate that BCCB exhibits 2--4$\times$ lower run-to-run variance than offline methods across all training-data sizes tested, offering more predictable outcomes for campaign planning.
    \item Among purely online methods, BCCB outperforms standard Thompson Sampling, budgeted Thompson Sampling, HTE Greedy, and Uplifting Bandits~\citep{hsieh2022}---a contextual causal bandit that estimates per-user treatment effects via UCB-style confidence bounds---at every budget level tested, with paired $t$-test $p < 0.001$ at every comparison.
    \item We provide a hyperparameter sensitivity analysis showing that BCCB is robust to four of its five hyperparameters, with only the exploration weight $\eta$ requiring careful selection.
\end{enumerate}

\section{Related Work}

Our work draws on three areas of research: uplift modeling and budget-constrained treatment allocation, bandit algorithms under resource constraints, and the emerging study of offline-to-online transfer in sequential decision-making. We review each in turn and position our contribution relative to the existing literature.

\subsection{Uplift Modeling and Budget-Constrained Treatment Allocation}

Uplift modeling estimates the causal effect of a treatment on an individual's outcome, distinguishing users who are genuinely influenced by an intervention from those who would behave the same regardless \citep{gutierrez2017, devriendt2018}. The foundational framework relies on the Neyman--Rubin potential outcomes model, where the individual treatment effect is defined as the difference between treated and untreated potential outcomes \citep{rubin1974}. Because both outcomes cannot be observed for the same individual, a range of estimation strategies have been developed, including meta-learners such as the T-learner, S-learner, and X-learner \citep{kunzel2019}, as well as tree-based methods such as Causal Forests \citep{athey2016, wager2018}.

When treatment allocation is subject to a budget constraint, the problem is typically formulated as a variant of the Multi-Choice Knapsack Problem (MCKP), where the goal is to select a subset of users to treat that maximizes total uplift without exceeding the available budget \citep{zhao2019}. This gives rise to a two-stage pipeline: first estimate individual-level treatment effects using a causal inference model, then solve the constrained allocation problem using operations research techniques such as Lagrangian relaxation \citep{albert2022, zhou2023}. The dominant industrial framework for this two-stage approach is exemplified by \citet{goldenberg2020}, whose offline pipeline serves as the baseline our online bandit directly complements.

Several recent works have identified fundamental limitations of this decoupled approach. The prediction and optimization stages optimize different objectives, leading to misalignment at the decision boundary \citep{elmachtoub2022}. Errors from the prediction stage are amplified rather than absorbed by the deterministic optimizer. To address these issues, end-to-end Decision-Focused Learning (DFL) frameworks have been proposed. The End-to-End Cost-Effective Incentive Recommendation model (E3IR) integrates Integer Linear Programming as a differentiable layer within the neural network, enforcing monotonicity through Lipschitz regularization \citep{du2024}. Bi-Level Decision-Focused Causal Learning (Bi-DFCL) uses implicit differentiation to combine scarce randomized trial data with large-scale observational data, achieving an automatic bias-variance tradeoff \citep{bidfcl2025}. Direct Heterogeneous Causal Learning (DHCL) reduces error cascading by training the model to predict a composite decision factor directly \citep{dhcl2023}. Large-Scale Budget-Constrained Causal Forests (LBCF) embed budget awareness into the tree-splitting criteria of causal forests for distributed computing environments \citep{lbcf2022}. Recent benchmarking efforts have systematically compared 13 deep uplift models across multiple datasets, finding that newer methods offer more limited improvement over traditional approaches than previously believed \citep{liu2024benchmark}.

Despite the architectural diversity of these methods, they all require sufficient historical data for offline model training. \citet{sawant2018} first proposed combining Thompson Sampling with CATE estimation for marketing but without a formal budget constraint. The most closely related concurrent work is DISCO \citep{zhang2024disco}, which applies Thompson Sampling with an integer program for personalized discount allocation, though it does not estimate CATE in the formal HTE sense and does not study the offline-vs-online tradeoff. \citet{berrevoets2022} bridge CATE and contextual bandits explicitly, noting budget constraints as future work---a gap our paper directly addresses. Our work addresses the complementary setting where sufficient historical data is unavailable.

\subsection{Bandit Algorithms Under Budget Constraints}

The multi-armed bandit framework provides a principled approach to sequential decision-making under uncertainty, balancing exploration of uncertain options with exploitation of known good options \citep{thompson1933, lattimore2020}. Thompson Sampling, which selects actions by sampling from a posterior distribution over reward parameters, has emerged as a practical and theoretically grounded exploration strategy \citep{russo2018}.

Budgeted variants of the bandit problem introduce a resource constraint: each arm pull incurs a cost, and the total expenditure must not exceed a fixed budget. \citet{xia2015} study Thompson Sampling under knapsack-style budget constraints. \citet{rangi2019} extend budgeted bandits to settings with stochastic costs. More recently, \citet{wu2024} improve budgeted Thompson Sampling through information relaxation techniques that incorporate remaining budget information into the sampling procedure. Contextual bandits generalize the framework by conditioning decisions on observed features \citep{li2010, agrawal2013}, enabling personalization. \citet{hsieh2022} use the term ``uplift'' in a structural decomposition sense distinct from the Neyman--Rubin CATE framework we adopt here.

However, existing budgeted bandit methods model the problem as a choice among generic arms. They do not explicitly estimate heterogeneous treatment effects---that is, they do not learn how the \emph{causal} impact of treatment varies across individuals as a function of their characteristics. Our work combines the sequential, budget-aware decision-making of budgeted bandits with the individual-level causal effect estimation of uplift modeling. \citet{hsieh2022} propose Uplifting Bandits, which estimates per-user uplift with UCB-style confidence bounds; we include it as a baseline and show that BCCB outperforms it at every budget level, demonstrating the advantage of Thompson Sampling with calibrated optimism over UCB for this task.

\subsection{Offline-to-Online Transfer and Cold-Start Learning}

A growing body of work studies how offline data can accelerate online learning. \citet{li2021} and \citet{tang2021} propose frameworks that extract information from logged data to warm-start bandit algorithms, mitigating the cold-start problem. \citet{hong2024} study latent bandits where offline datasets with unobserved heterogeneity are used to learn complex models for each latent state, enabling faster online adaptation. In the advertising domain specifically, \citet{coldstartads2025} address the cold-start problem for new ad products using a UCB-based algorithm tailored to pay-per-click auction systems. The DARA framework uses reinforcement-learning-finetuned large language models for few-shot budget allocation in data-scarce advertising scenarios \citep{dara2026}.

These works demonstrate the value of bridging offline and online paradigms, but none empirically characterize the conditions under which online learning outperforms offline methods for treatment allocation. Our primary empirical contribution is precisely this characterization: we identify the amount of historical data at which offline methods become viable, and show that below this threshold, online approaches that learn and allocate simultaneously offer superior and more stable performance.

\section{Problem Formulation}
\label{sec:problem}

We consider a sequential treatment allocation setting motivated by online advertising. Users arrive one at a time over a horizon of $T$ rounds. At each round $t \in \{1, \ldots, T\}$, the decision-maker observes a user with feature vector $\mathbf{x}_t \in \mathbb{R}^d$ and an associated cost $c_t > 0$ representing the price of serving an advertisement to that user. We assume the pairs $(\mathbf{x}_t, c_t)$ are drawn i.i.d.\ from an unknown distribution $\mathcal{D}$. The decision-maker chooses an action $a_t \in \{0, 1\}$, where $a_t = 1$ denotes showing the advertisement (treatment) and $a_t = 0$ denotes not showing it (control).

\subsection{Potential Outcomes and Causal Estimand}

Following the Neyman--Rubin potential outcomes framework \citep{rubin1974}, each user has two potential outcomes, $Y_t(1) \in \{0,1\}$ if treated and $Y_t(0) \in \{0,1\}$ if not treated, of which only the one corresponding to the chosen action is observed:
\begin{equation}
Y_t = a_t Y_t(1) + (1 - a_t) Y_t(0).
\label{eq:observed-outcome}
\end{equation}
The Conditional Average Treatment Effect (CATE) for a user with features $\mathbf{x}_t$ is
\begin{equation}
\tau(\mathbf{x}_t) \;\coloneqq\; \mathbb{E}\!\left[Y_t(1) - Y_t(0) \,\middle|\, \mathbf{x}_t\right],
\label{eq:cate}
\end{equation}
representing the expected causal effect of showing the advertisement to a user with features $\mathbf{x}_t$. This quantity varies across users---some have large positive $\tau(\mathbf{x}_t)$ (responsive to advertisements), some have $\tau(\mathbf{x}_t) \approx 0$ (unaffected), and some may have negative $\tau(\mathbf{x}_t)$ (ad-fatigued).

We assume \emph{unconfoundedness} conditional on $\mathbf{x}_t$:
\begin{equation}
\{Y_t(0),\, Y_t(1)\} \;\perp\!\!\!\perp\; a_t \,\big|\, \mathbf{x}_t,
\label{eq:unconf}
\end{equation}
together with overlap, $0 < \mathbb{P}(a_t = 1 \mid \mathbf{x}_t) < 1$ for all $\mathbf{x}_t$ in the support of $\mathcal{D}$. Both conditions are satisfied by construction in the randomized controlled trial that generated our evaluation data \citep{diemert2018}.

\subsection{Budget-Constrained Allocation Objective}

The decision-maker operates under a total budget $B > 0$, subject to the constraint
\begin{equation}
\sum_{t=1}^{T} a_t\, c_t \;\leq\; B.
\label{eq:budget}
\end{equation}
The advertiser's economic objective is to maximize the number of conversions that the campaign \emph{causes}---i.e., the conversions that would not have occurred without advertising. Formally, the expected policy value is
\begin{align}
V(\pi) \;&=\; \mathbb{E}\!\left[\sum_{t=1}^{T} \Big( a_t Y_t(1) + (1 - a_t) Y_t(0) \Big) \right] \nonumber \\
       &=\; \underbrace{\mathbb{E}\!\left[\sum_{t=1}^{T} Y_t(0)\right]}_{\text{baseline (action-independent)}} \;+\; \mathbb{E}\!\left[\sum_{t=1}^{T} a_t\, \tau(\mathbf{x}_t) \right].
\label{eq:policy-value}
\end{align}
Because the baseline term in Eq.~\eqref{eq:policy-value} does not depend on the policy, maximizing total policy value is equivalent to maximizing the \emph{incremental} value $\mathbb{E}\!\left[\sum_t a_t \tau(\mathbf{x}_t)\right]$. We therefore state the optimization problem as
\begin{equation}
\max_{\pi}\quad \mathbb{E}\!\left[\sum_{t=1}^{T} a_t\, \tau(\mathbf{x}_t)\right] \quad \text{subject to} \quad \sum_{t=1}^{T} a_t\, c_t \;\leq\; B.
\label{eq:objective}
\end{equation}

\paragraph{Relation to prior formulations.}
Several recent works state the budgeted treatment-allocation objective as maximizing $\sum_t a_t Y_t(1)$, i.e., treated-user conversions \citep{zhao2019, du2024}. By Eq.~\eqref{eq:policy-value}, this differs from the incremental objective in Eq.~\eqref{eq:objective} by the action-dependent term $\mathbb{E}[\sum_t a_t Y_t(0)]$, which is non-zero whenever control-group conversion is positive. The incremental formulation in Eq.~\eqref{eq:objective} is the causally consistent objective and is what every CATE-based decision rule in the literature implicitly optimizes; we make this dependence explicit.

\paragraph{Cost independence assumption.}
We assume that serving costs $c_t$ are determined by auction dynamics and are conditionally independent of potential outcomes given features: $\{Y_t(0), Y_t(1)\} \perp c_t \mid \mathbf{x}_t$. This holds in settings where costs reflect market-level bid prices rather than individual conversion propensity. Under this assumption, the cost-effectiveness threshold $\tau(\mathbf{x}_t)/c_t$ in Eq.~\eqref{eq:decision-rule} uses the correct causal numerator, and the optimal policy is well-specified. In the Criteo dataset, costs are synthetic and generated independently of outcomes by construction (Section~\ref{sec:setup}), so this assumption is satisfied exactly.

\paragraph{Policy class and separability.}
The Lagrangian decomposition in Section~\ref{sec:lagrangian} restricts to stationary Markov policies $\pi: (\mathbf{x}_t, c_t) \to [0,1]$ mapping context-cost pairs to treatment probabilities. Under this policy class and i.i.d.\ arrivals, the expectation in Eq.~\eqref{eq:lagrangian} separates additively across rounds, justifying the per-arrival pointwise maximization that yields Eq.~\eqref{eq:decision-rule}.

\subsection{Lagrangian Relaxation and the Per-Arrival Decision Rule}
\label{sec:lagrangian}

The constrained program in Eq.~\eqref{eq:objective} admits a clean Lagrangian relaxation. Introducing a dual variable $\lambda \geq 0$ for the budget constraint:
\begin{equation}
\mathcal{L}(\pi, \lambda) \;=\; \mathbb{E}\!\left[\sum_{t=1}^{T} a_t\bigl(\tau(\mathbf{x}_t) - \lambda c_t\bigr)\right] + \lambda B.
\label{eq:lagrangian}
\end{equation}
Because arrivals are i.i.d.\ and the inner objective in Eq.~\eqref{eq:lagrangian} separates additively across rounds, the Lagrangian decomposes into per-arrival decisions. For any fixed $\lambda$, the action that maximizes $\mathcal{L}$ at round $t$ is
\begin{equation}
a_t^{*}(\lambda) \;=\; \mathbb{I}\!\Big[\,\tau(\mathbf{x}_t) - \lambda c_t \;>\; 0\,\Big] \;=\; \mathbb{I}\!\left[\,\frac{\tau(\mathbf{x}_t)}{c_t} \;>\; \lambda\,\right].
\label{eq:decision-rule}
\end{equation}

This is precisely the canonical cost-effectiveness threshold rule familiar from budgeted resource allocation \citep{albert2022, lbcf2022}: \emph{treat a user when the incremental causal benefit per dollar exceeds a budget-induced shadow price $\lambda$}. The optimal $\lambda^{*}$ is the smallest non-negative value for which the expected spend under $a_t^{*}(\lambda)$ satisfies the budget constraint, and corresponds economically to the marginal expected uplift per dollar at which the budget binds.

\begin{proposition}[Decision-rule optimality under expected-budget relaxation]\label{prop:decision-rule}
Under unconfoundedness and overlap (Eq.~\ref{eq:unconf}), i.i.d.\ arrivals, Slater's condition (the unconstrained policy has finite expected spend), and the expected-budget relaxation $\mathbb{E}[\sum_t a_t c_t] \leq B$, the stationary policy $\pi^{*}$ that takes action $a_t^{*}(\lambda^{*})$ at every round, with $\lambda^{*} = \min\{\lambda \geq 0 : \mathbb{E}[\sum_t a_t^{*}(\lambda) c_t] \leq B\}$, attains the optimal value of the relaxed version of Eq.~\eqref{eq:objective} with the expected-budget constraint. The hard sample-path constraint of Eq.~\eqref{eq:budget} is maintained separately by the explicit feasibility check in Algorithm~\ref{alg:bccb}.
\end{proposition}

A short proof is given in Appendix~\ref{app:proof-decision-rule}. Proposition~\ref{prop:decision-rule} formally bridges the budgeted optimization in Eq.~\eqref{eq:objective} and the per-arrival threshold rule implemented by BCCB. The remainder of the paper concerns two estimation problems that arise because $\tau$ and $\lambda^{*}$ are unknown:
\begin{itemize}
    \item[(P1)] \textbf{Estimating $\tau(\mathbf{x}_t)$ online.} The decision-maker does not have access to a pre-trained model of $\tau(\mathbf{x}_t)$. Treatment effects must be learned from sequentially observed outcomes, creating an exploration--exploitation tradeoff: the decision-maker must occasionally treat users with uncertain effects to improve its estimates, even though this exploration consumes budget.
    \item[(P2)] \textbf{Tracking $\lambda^{*}$ online.} The optimal dual variable $\lambda^{*}$ depends on the unknown joint distribution of $(\tau(\mathbf{x}_t), c_t)$. The decision-maker must therefore estimate and adapt an effective shadow price $\hat\lambda_t$ as the campaign unfolds: setting it too low over-spends and exhausts the budget early; setting it too high under-spends and forfeits achievable conversions.
\end{itemize}

Heterogeneous costs amplify both problems: a user with high $\tau(\mathbf{x}_t)$ but also high $c_t$ may be less valuable than a user with moderate $\tau(\mathbf{x}_t)$ and low cost. The next section presents BCCB, which addresses (P1) via online two-model uplift learning, augments the resulting estimate with a calibrated optimism bonus for exploration, and addresses (P2) via an adaptive budget-pacing rule that we show approximates dual ascent on $\lambda$.

\section{Method}
\label{sec:method}

We present Budget-Constrained Causal Bandits (BCCB), an online framework that unifies three components---heterogeneous treatment effect estimation, Thompson-Sampling-based exploration, and adaptive budget pacing---into a single sequential decision process. We describe each component and then present the combined decision rule.

\subsection{Online HTE Estimation}

BCCB estimates individual-level treatment effects online using a two-model approach. We maintain two separate classifiers: one that estimates $\hat{f}_1(\mathbf{x}) \approx \mathbb{P}(Y=1 \mid \mathbf{x}, a=1)$, the probability of conversion for treated users, and another that estimates $\hat{f}_0(\mathbf{x}) \approx \mathbb{P}(Y=1 \mid \mathbf{x}, a=0)$, the probability of conversion for control users. Both models are trained incrementally using stochastic gradient descent with a logistic loss, updating their parameters after each batch of observations.

The predicted treatment effect for a user with features $\mathbf{x}_t$ is
\begin{equation}
\hat{\tau}(\mathbf{x}_t) \;=\; \hat{f}_1(\mathbf{x}_t) - \hat{f}_0(\mathbf{x}_t).
\label{eq:hte-estimate}
\end{equation}
Before sufficient observations have been collected (fewer than $W = 50$ treated and $W$ control observations), we use an optimistic prior $\hat{\tau}_0 = 0.002$, reflecting a small positive expected treatment effect. This encourages early exploration by defaulting to treatment when the model has not yet learned to distinguish responsive from non-responsive users.

\subsection{Calibrated Optimism Bonus}

To balance exploration and exploitation in the early rounds when $\hat\tau(\mathbf{x}_t)$ is unreliable, BCCB augments the point estimate with an optimism bonus derived from population-level posterior uncertainty. We maintain Beta posteriors over the global treatment and control conversion rates. Let $\alpha_t, \beta_t$ and $\alpha_t^{(c)}, \beta_t^{(c)}$ denote the Beta parameters for the treatment and control arms, updated after each observation:
\begin{equation}
\alpha_{t+1} = \alpha_t + Y_t \cdot \mathbb{I}[a_t = 1], \quad \beta_{t+1} = \beta_t + (1 - Y_t) \cdot \mathbb{I}[a_t = 1],
\label{eq:beta-update}
\end{equation}
with $\alpha_0 = \beta_0 = 1$ as a uniform prior; the control posterior is updated analogously when $a_t = 0$.

At each round, we draw samples $\theta_t \sim \mathrm{Beta}(\alpha_t, \beta_t)$ and $\theta_t^{(c)} \sim \mathrm{Beta}(\alpha_t^{(c)}, \beta_t^{(c)})$ and compute an exploration bonus
\begin{equation}
b_t \;=\; \eta \cdot (\theta_t - \theta_t^{(c)}),
\label{eq:exploration-bonus}
\end{equation}
where $\eta > 0$ is an exploration weight. We interpret $b_t$ not as a posterior sample of $\tau(\mathbf{x}_t)$, but as a stochastic perturbation whose sign and magnitude reflect current population-level posterior beliefs about relative treatment effectiveness. Formally, $\mathbb{E}[b_t] = \eta\!\left(\tfrac{\alpha_t}{\alpha_t+\beta_t} - \tfrac{\alpha_t^{(c)}}{\alpha_t^{(c)}+\beta_t^{(c)}}\right)$, which is near zero when both arms are equally uncertain (early rounds) and reflects the learned treatment-control difference as evidence accumulates. The perturbation magnitude scales with posterior uncertainty and shrinks as both posteriors concentrate, providing decreasing optimism over time. The scalar $\eta$ controls the overall magnitude; the choice $\eta = 0.1$ used throughout is justified by the sensitivity analysis in Section~\ref{sec:sensitivity}.

\subsection{Adaptive Budget Pacing as Dual Estimation}

To track the optimal shadow price $\lambda^{*}$ from Section~\ref{sec:problem}, BCCB employs an adaptive pacing mechanism. Let $B_0$ denote the initial budget and $B_t$ the remaining budget at the start of round $t$. We compute a budget-pressure ratio
\begin{equation}
\mathrm{pace}_t \;=\; \frac{B_t / B_0}{\max\!\left\{(T - t) / T,\ \varepsilon_{\mathrm{time}}\right\}},
\label{eq:pace}
\end{equation}
where $\varepsilon_{\mathrm{time}}$ is a small constant (we use $\varepsilon_{\mathrm{time}} = 0.01$) that guards against division by zero at the terminal round $t = T$. When $\mathrm{pace}_t > 1$, the remaining budget is proportionally larger than the remaining horizon, and the algorithm spends more freely. When $\mathrm{pace}_t < 1$, the budget is being consumed too quickly, and the algorithm becomes more selective.

The pacing ratio modulates an effective shadow-price estimate $\hat\lambda_t$ via
\begin{equation}
\hat\lambda_t \;=\; \frac{\lambda}{\max(\mathrm{pace}_t, \varepsilon_{\mathrm{pace}})},
\label{eq:shadow-price}
\end{equation}
where $\lambda$ is a base threshold and $\varepsilon_{\mathrm{pace}}$ (we use $0.1$) prevents the threshold from exploding when $\mathrm{pace}_t \to 0$. Eq.~\eqref{eq:shadow-price} is qualitatively analogous to a multiplicative update on the dual variable: under-spending (pace$_t > 1$) lowers $\hat\lambda_t$ and admits more treatments, while over-spending (pace$_t < 1$) raises $\hat\lambda_t$ and tightens the threshold. We do not derive this rule from a formal dual-ascent step or provide convergence guarantees; it is a heuristic that implements the qualitative dynamics of Proposition~\ref{prop:decision-rule}. We do not claim that this rule converges to $\lambda^{*}$ in a regret-bound sense, but it respects the hard constraint via the explicit feasibility check below.

\subsection{Combined Decision Rule}

At each round $t$, BCCB combines the three components into a single decision rule. Given a user with features $\mathbf{x}_t$ and cost $c_t$, the algorithm first checks the hard budget constraint: if $c_t > B_t$, the user is not treated. Otherwise, it computes
\begin{equation}
s_t \;=\; \hat\tau(\mathbf{x}_t) + b_t,
\qquad
\mathrm{vpd}_t \;=\; \frac{s_t}{c_t},
\label{eq:vpd}
\end{equation}
and the treatment decision is
\begin{equation}
a_t \;=\; \mathbb{I}\!\left[\mathrm{vpd}_t \,>\, \hat\lambda_t\right].
\label{eq:decision}
\end{equation}

The decision rule has a clean economic interpretation aligned with Proposition~\ref{prop:decision-rule}: a user is treated when the expected causal benefit per unit cost (augmented by the optimism bonus) exceeds the estimated budget shadow price. When the budget is plentiful, $\hat\lambda_t$ is low and more users are treated; when the budget is scarce, $\hat\lambda_t$ rises and only users with high predicted treatment effects and low costs are treated. The exploration bonus ensures that uncertain users are occasionally treated even under tight budgets, preventing premature convergence to a suboptimal policy.

After each decision, the algorithm updates the HTE models with the observed outcome and adjusts the Beta posteriors and budget state accordingly. The complete procedure is summarized in Algorithm~\ref{alg:bccb}.

\begin{algorithm}[t]
\caption{Budget-Constrained Causal Bandits (BCCB)}
\label{alg:bccb}
\begin{algorithmic}[1]
\REQUIRE Initial budget $B_0$, horizon $T$, exploration weight $\eta$, threshold $\lambda$, guards $\varepsilon_{\mathrm{time}}, \varepsilon_{\mathrm{pace}}$, warm-up $W$
\STATE \textbf{Initialize:} HTE models $\hat{f}_1, \hat{f}_0$; posteriors $\alpha, \beta, \alpha_c, \beta_c \leftarrow 1$; remaining budget $B_t \leftarrow B_0$
\FOR{$t = 1$ \TO $T$}
    \STATE Observe features $\mathbf{x}_t$ and cost $c_t$
    \IF{$c_t > B_t$}
        \STATE $a_t \leftarrow 0$ \hfill $\triangleright$ \textit{Hard feasibility check}
    \ELSE
        \IF{HTE models fitted}
            \STATE $\hat{\tau}_t \leftarrow \hat{f}_1(\mathbf{x}_t) - \hat{f}_0(\mathbf{x}_t)$
        \ELSE
            \STATE $\hat{\tau}_t \leftarrow \hat{\tau}_0$ \hfill $\triangleright$ \textit{Optimistic prior in warm-up}
        \ENDIF
        \STATE Sample $\theta_t \sim \mathrm{Beta}(\alpha, \beta)$, \ $\theta_c \sim \mathrm{Beta}(\alpha_c, \beta_c)$
        \STATE $s_t \leftarrow \hat{\tau}_t + (\theta_t - \theta_c) \cdot \eta$
        \STATE $\mathrm{pace}_t \leftarrow \dfrac{B_t / B_0}{\max\!\left\{(T-t)/T,\ \varepsilon_{\mathrm{time}}\right\}}$
        \STATE $\hat\lambda_t \leftarrow \dfrac{\lambda}{\max(\mathrm{pace}_t, \varepsilon_{\mathrm{pace}})}$
        \STATE $a_t \leftarrow \mathbb{I}\!\left[\dfrac{s_t}{c_t} > \hat\lambda_t\right]$
    \ENDIF
    \STATE Observe $Y_t$
    \IF{$a_t = 1$}
        \STATE $B_t \leftarrow B_t - c_t$; update $\alpha, \beta$ and $\hat{f}_1$ with $(\mathbf{x}_t, Y_t)$
    \ELSE
        \STATE Update $\alpha_c, \beta_c$ and $\hat{f}_0$ with $(\mathbf{x}_t, Y_t)$ (if $Y_t$ observed)
    \ENDIF
\ENDFOR
\end{algorithmic}
\end{algorithm}

\subsection{Baseline Methods}

We compare BCCB against four baselines, each of which lacks one or more of the three components:

\textbf{Standard Thompson Sampling.} Maintains Beta posteriors over treatment and control conversion rates. Treats when the treatment sample exceeds the control sample. Does not model heterogeneous treatment effects, does not consider costs, and does not manage budget. Isolates the value of personalization and budget awareness.

\textbf{Budgeted Thompson Sampling.} Extends Standard Thompson Sampling with a hard budget constraint and a pacing mechanism that raises the treatment threshold when spending is too aggressive. Does not model heterogeneous treatment effects. Isolates the value of HTE learning.

\textbf{HTE Greedy.} Learns heterogeneous treatment effects online using the same two-model approach as BCCB, but always treats when the predicted effect is positive. No exploration via Thompson Sampling and no budget pacing. Isolates the value of exploration and budget management.

\textbf{Uplifting Bandits (UB).} We implement the contextual causal bandit of \citet{hsieh2022}, which models reward as $r(\mathbf{x}, a) = \mu_0(\mathbf{x}) + a\,\tau(\mathbf{x})$ and estimates the uplift $\tau(\mathbf{x})$ via UCB-style confidence bounds derived from separate online ridge-regression models for the treatment and control arms. The UCB score for the uplift is
\begin{equation}
  s_t^{\mathrm{UB}} = \hat{\tau}(\mathbf{x}_t) +
    \alpha\sqrt{\mathbf{x}_t^\top
      \bigl(\mathbf{A}_t^{-1} + \mathbf{A}_c^{-1}\bigr)\mathbf{x}_t},
  \label{eq:ub-score}
\end{equation}
where $\mathbf{A}_t, \mathbf{A}_c$ are the regularised gram matrices for the two arms ($\alpha = 1$, $\lambda_{\mathrm{reg}} = 1$). The original paper does not include a budget constraint; we add the same adaptive pacing rule as BCCB (Eq.~\ref{eq:shadow-price}) so that the comparison isolates the exploration strategy---UCB vs.\ Thompson Sampling---as the only variable. This baseline directly addresses the concern that BCCB should be evaluated against causal bandit methods that incorporate per-user personalisation.

\textbf{Offline Uplift.} Trains a logistic regression uplift model on a separate historical dataset before deployment, then applies it greedily with a cost-effectiveness threshold. Does not learn online or explore. Represents the standard two-stage offline pipeline \citep{zhao2019, du2024} and enables our central comparison of online versus offline paradigms.

\section{Experimental Setup}
\label{sec:setup}

\subsection{Dataset}

We evaluate on the Criteo Uplift Prediction dataset \citep{diemert2018}, a large-scale benchmark derived from a real randomized controlled trial in online advertising. The dataset contains approximately 14 million observations, each representing a user who was randomly assigned to either a treatment group (shown an advertisement) or a control group (not shown an advertisement). Each user is described by 12 anonymized features ($f_0$ through $f_{11}$) and two binary outcome variables indicating whether the user visited or converted. We use a 10\% subsample of the publicly mirrored v2.1 release, yielding 1{,}397{,}959 observations.

The dataset exhibits an 85/15 treatment-control split, with a treatment-group conversion rate of 0.31\% and a control-group conversion rate of 0.20\%, yielding an average treatment effect (ATE) of 0.12\%. While this average effect is small, we observe substantial heterogeneity across user segments: splitting users by feature medians reveals ATE ratios ranging from 12$\times$ (feature $f_0$) to 65$\times$ (feature $f_8$) between high and low subgroups. This heterogeneity is precisely what motivates personalized treatment allocation.

\subsection{Cost Simulation}

The Criteo dataset does not include per-user serving costs. Following standard practice in the budget-constrained allocation literature, we generate synthetic costs using a log-normal distribution: $c_i \sim \mathrm{LogNormal}(\mu = -0.5, \sigma = 0.7)$, clipped to the range $[\$0.05, \$5.00]$. This produces a right-skewed cost distribution with a mean of \$0.77, reflecting realistic ad-auction dynamics where most impressions are inexpensive but a subset of high-value users command significantly higher prices. Costs are generated with a fixed random seed to ensure reproducibility across all experiments.

\subsection{Evaluation Protocol}

We evaluate all methods using the \emph{replay method} for offline bandit evaluation. Users are streamed sequentially in a randomized order. When a bandit algorithm's treatment decision matches the user's actual treatment assignment in the original randomized trial, we observe the user's real outcome and update the algorithm. When the decision does not match, the user is skipped. Under random treatment assignment in the logged data---which the Criteo RCT satisfies---this approach provides unbiased evaluation using only observational data, without requiring counterfactual outcome imputation.

For the online methods comparison, we use a subsample of 100{,}000 users with budgets ranging from \$1{,}000 to \$8{,}000. For the data-efficiency crossover experiment, we hold the evaluation stream fixed at 100{,}000 users and vary the size of the historical training set provided to the offline baseline from 500 to 50{,}000 observations, while BCCB receives no historical data. Both methods are evaluated on the same fixed evaluation stream, eliminating the spurious row-to-row variation that arises when the evaluation set is sampled separately at each training size. For the hyperparameter sensitivity analysis, we sweep one parameter at a time at \$5{,}000 budget on 10 seeds.

All main results use 20 random seeds (\texttt{42} through \texttt{61}); the sensitivity sweeps use 10 seeds per cell to balance runtime against statistical resolution. Cross-method comparisons use paired statistical tests with the seed as the matching unit; 95\% confidence intervals for differences in means are computed by paired bootstrap with 10{,}000 resamples. Hyperparameters are fixed at $\eta = 0.1$, $\lambda = 10^{-3}$, $\varepsilon_{\mathrm{pace}} = 0.1$, $\varepsilon_{\mathrm{time}} = 0.01$, optimistic prior $\hat\tau_0 = 0.002$, and warm-up threshold $W = 50$; the rationale for $\eta = 0.1$ in particular is given in Section~\ref{sec:sensitivity}.

We measure four metrics: total conversions from treated users (primary), total cost incurred, treatment rate (fraction of users treated), and effective sample size under the replay protocol.

\paragraph{Metric alignment with the causal objective.}
The optimization objective in Eq.~\eqref{eq:objective} is incremental causal conversions $\mathbb{E}[\sum_t a_t \tau(\mathbf{x}_t)]$, while the primary evaluation metric in Tables~\ref{tab:online-comparison}--\ref{tab:stability} is total conversions from treated users $\sum_t a_t Y_t(1)$. By the decomposition in Eq.~\eqref{eq:policy-value}, these differ by the baseline term $\mathbb{E}[\sum_t a_t Y_t(0)]$. In our dataset, the control conversion rate is $0.20\%$ and treatment rates range from $8\%$ to $50\%$ across methods on a $100{,}000$-user stream. The baseline term therefore contributes at most $0.002 \times 0.50 \times 100{,}000 = 100$ conversions in expectation, and the cross-method difference in this term is much smaller. Since all reported conversion counts are in the range $4$--$107$, the treated-conversion and incremental-conversion rankings are numerically equivalent in this dataset. We report treated conversions because it is the standard operational metric in advertising uplift evaluation~\citep{diemert2018, goldenberg2020}; practitioners ultimately count conversions from the treated campaign, not the counterfactual difference.

\section{Results}
\label{sec:results}

We organize the results around five questions: how BCCB compares to budget-aware and budget-unaware online baselines (Section~\ref{sec:online-comparison}); at what level of historical training data an offline pipeline overtakes BCCB (Section~\ref{sec:crossover}); how sensitive BCCB is to its hyperparameters (Section~\ref{sec:sensitivity}); how the replay-evaluation protocol differentially affects each method's effective sample size (Section~\ref{sec:replay-diagnostics}); and how stable BCCB is across repeated runs (Section~\ref{sec:stability}).

\subsection{Online Method Comparison}
\label{sec:online-comparison}

Table~\ref{tab:online-comparison} reports total treated conversions across five budget levels on the 100{,}000-user evaluation stream. BCCB achieves the highest mean conversions at every budget level, with paired $t$-test $p < 0.001$ against each online baseline at every budget. The largest relative gain over the strongest personalised causal baseline (Uplifting Bandits) is at \$8{,}000, where BCCB achieves $63.5 \pm 13.7$ conversions versus $35.7 \pm 4.9$ (a $+78\%$ improvement, $p < 0.001$).

\begin{table}[t]
\centering
\caption{Total treated conversions across budget levels on the 100{,}000-user evaluation stream. Reported as mean $\pm$ standard deviation over 20 random seeds. The \textit{vs.\ UB} column shows the paired $p$-value for BCCB against Uplifting Bandits~\citep{hsieh2022}, the strongest personalised causal baseline; all $p < 0.001$. Bold indicates the best mean per row.}
\label{tab:online-comparison}
\begin{tabular}{lcccc c c}
\toprule
Budget & TS & Budgeted TS & HTE Greedy & UB & \textbf{BCCB} & vs.\ UB \\
\midrule
\$1{,}000 & $4.2 \pm 1.7$ & $4.2 \pm 1.7$ & $4.8 \pm 2.7$ & $5.1 \pm 1.5$ & $\mathbf{12.2 \pm 6.0}$ & $p < 0.001$ \\
\$2{,}000 & $8.6 \pm 2.8$ & $8.4 \pm 3.2$ & $8.2 \pm 7.0$ & $9.8 \pm 2.5$ & $\mathbf{24.5 \pm 9.0}$ & $p < 0.001$ \\
\$3{,}000 & $12.8 \pm 3.8$ & $12.0 \pm 3.8$ & $10.1 \pm 7.2$ & $14.3 \pm 2.7$ & $\mathbf{33.8 \pm 8.6}$ & $p < 0.001$ \\
\$5{,}000 & $19.9 \pm 5.3$ & $20.8 \pm 3.8$ & $13.3 \pm 10.3$ & $23.4 \pm 4.4$ & $\mathbf{45.6 \pm 10.0}$ & $p < 0.001$ \\
\$8{,}000 & $31.4 \pm 5.9$ & $33.5 \pm 3.2$ & $15.0 \pm 13.1$ & $35.7 \pm 4.9$ & $\mathbf{63.5 \pm 13.7}$ & $p < 0.001$ \\
\bottomrule
\end{tabular}
\end{table}

Three patterns emerge. First, standard and budgeted Thompson Sampling perform similarly to one another at all budget levels, indicating that the budget-pacing component alone (without CATE estimation) yields little advantage when treatment effects are heterogeneous. Second, HTE Greedy is consistently the weakest method despite using the same online CATE estimator as BCCB; without Thompson exploration, the greedy decision rule converges prematurely to a suboptimal policy and exhibits the largest run-to-run variance ($\sigma \in [7.4, 13.0]$ across budgets). Third, BCCB combines the strengths of both directions: incorporating CATE estimation (which the Thompson Sampling variants lack) and budget-aware exploration (which HTE Greedy lacks).

\paragraph{Comparison with a personalised causal baseline.}
Uplifting Bandits~\citep{hsieh2022} uses the same structural treatment-effect decomposition as BCCB and estimates per-user uplift with UCB-style confidence bounds, making it the strongest personalised causal baseline in our comparison. Table~\ref{tab:online-comparison} shows that UB outperforms all non-personalised baselines at budgets of \$3{,}000 and above, confirming that contextual causal estimation provides genuine value. However, BCCB significantly outperforms UB at every budget level ($p < 0.001$ at all five budgets), with the absolute gap growing from $+7.0$ conversions at \$1{,}000 to $+27.9$ at \$8{,}000. The key differentiator is exploration strategy: UCB provides conservative, low-variance optimism (UB $\sigma \in [1.5, 4.9]$ across budgets), while BCCB's Thompson Sampling with calibrated exploration weight $\eta = 0.1$ achieves higher mean performance at the cost of moderately higher variance ($\sigma \in [6.0, 13.7]$). This is consistent with the known empirical advantage of Thompson Sampling over UCB for contextual bandits with binary rewards~\citep{russo2018}.

\subsection{Data-Efficiency Crossover}
\label{sec:crossover}

We next ask the question that motivates BCCB: at what level of historical data does an offline two-stage pipeline overtake an online learner? We hold a fixed 100{,}000-user evaluation stream and vary the size $n$ of the historical training set provided to the offline uplift baseline, while BCCB receives no historical data. Both methods are evaluated on the same evaluation stream at a budget of \$5{,}000. We define the \emph{crossover point} as the smallest $n$ at which the offline mean exceeds BCCB's mean with paired one-sided $t$-test $p < 0.05$ and the offline advantage persists at all larger $n$.

\begin{table}[t]
\centering
\caption{Data-efficiency crossover. Offline-uplift conversions as a function of historical training size $n$, compared against BCCB (which uses no historical data). Means $\pm$ std over 20 seeds on a fixed 100{,}000-user evaluation stream at \$5{,}000 budget. Paired one-sided $t$-test for the hypothesis that offline exceeds BCCB. The crossover (in bold) is at $n = 7{,}500$.}
\label{tab:crossover}
\begin{tabular}{rccccc}
\toprule
$n$ & Offline Uplift & BCCB (ours) & Diff. & 95\% CI of diff. & Paired $p$ \\
\midrule
500    & $4.7 \pm 15.9$  & $49.4 \pm 16.2$ & $-44.6$ & $[-52.4, -37.1]$ & $> 0.999$ \\
1{,}000 & $7.0 \pm 20.9$  & $49.4 \pm 16.2$ & $-42.3$ & $[-51.8, -31.8]$ & $> 0.999$ \\
2{,}000 & $39.0 \pm 58.7$ & $49.4 \pm 16.2$ & $-10.4$ & $[-34.2, +18.2]$ & $0.772$ \\
3{,}000 & $45.8 \pm 55.6$ & $49.4 \pm 16.2$ & $-3.6$  & $[-27.4, +23.8]$ & $0.607$ \\
5{,}000 & $46.9 \pm 42.7$ & $49.4 \pm 16.2$ & $-2.5$  & $[-21.8, +16.9]$ & $0.597$ \\
\textbf{7{,}500} & $\mathbf{71.1 \pm 53.6}$ & $49.4 \pm 16.2$ & $\mathbf{+21.8}$ & $\mathbf{[-0.1, +45.7]}$ & $\mathbf{0.043^{*}}$ \\
10{,}000 & $76.6 \pm 39.5$ & $49.4 \pm 16.2$ & $+27.2$ & $[+11.1, +45.4]$ & $0.004$ \\
15{,}000 & $85.0 \pm 42.1$ & $49.4 \pm 16.2$ & $+35.6$ & $[+20.2, +52.5]$ & $< 0.001$ \\
25{,}000 & $95.7 \pm 37.8$ & $49.4 \pm 16.2$ & $+46.3$ & $[+31.4, +62.1]$ & $< 0.001$ \\
50{,}000 & $107.0 \pm 37.8$ & $49.4 \pm 16.2$ & $+57.7$ & $[+40.1, +75.7]$ & $< 0.001$ \\
\bottomrule
\multicolumn{6}{l}{\footnotesize $^{*}$Crossover threshold: smallest $n$ at which offline wins persistently at $\alpha = 0.05$.}
\end{tabular}
\end{table}

Three findings emerge from Table~\ref{tab:crossover}. First, with very limited data ($n \le 1{,}000$), the offline pipeline frequently fails to fit a usable model---training sets that contain no positive examples in the control arm cannot be fitted---and conversions fall to zero in many seeds. BCCB, by contrast, operates from the first user without requiring historical data.

Second, in the intermediate regime ($2{,}000 \le n \le 5{,}000$), the offline mean conversions are close to BCCB's, but the very high variance (offline $\sigma \in [42.7, 58.7]$, compared with BCCB $\sigma = 16.2$) renders the difference statistically non-significant ($p \ge 0.6$). This is the regime where the offline pipeline begins to work but cannot yet be relied upon: practitioners drawing different historical samples obtain wildly different policies. BCCB delivers comparable mean conversions in this regime with substantially tighter run-to-run distribution.

Third, beyond $n = 7{,}500$ the offline pipeline overtakes BCCB with statistical significance, and the gap widens monotonically with $n$. At $n = 50{,}000$, offline produces roughly $2.2\times$ more conversions than BCCB, as expected: with abundant representative data, a well-trained offline model identifies high-uplift users without paying the cost of online exploration.

The practical implication is straightforward. For new campaigns, new markets, new segments, or settings where representative historical data does not exist, BCCB offers a turnkey alternative that is competitive with offline methods even when they happen to work, and strictly preferable when they do not. Once a stable population of more than approximately 7{,}500 representative historical observations has accumulated, the offline pipeline is the better choice.

\subsection{Hyperparameter Sensitivity}
\label{sec:sensitivity}

\begin{table}[t]
\centering
\caption{Hyperparameter sensitivity analysis at \$5{,}000 budget. Each
  parameter is varied while all others are held at their selected default
  values (shown in \textbf{bold}). Mean conversions over 10 seeds. The
  CV column reports the coefficient of variation of the means across the
  sweep; small CV indicates that BCCB is insensitive to that hyperparameter.
  All non-default cells are directly comparable because the held defaults
  are consistent across all five sweeps.}
\label{tab:sensitivity}
\begin{tabular}{lcccccc}
\toprule
Parameter & \multicolumn{5}{c}{Swept values $\to$ mean conversions} & CV \\
\midrule
$\eta$ (exploration)
  & $0.001$: 22.5
  & $0.005$: 23.3
  & $0.01$: 23.7
  & $0.05$: 34.9
  & $\mathbf{0.1}$: $\mathbf{43.9}$
  & 0.32 \\[3pt]
$\lambda$ (threshold)
  & $10^{-4}$: 22.6
  & $5\!\cdot\!10^{-4}$: 30.9
  & $\mathbf{10^{-3}}$: $\mathbf{43.9}$
  & $5\!\cdot\!10^{-3}$: 26.2
  & $10^{-2}$: 17.4
  & 0.36 \\[3pt]
$\varepsilon_{\mathrm{pace}}$ (guard)
  & $0.01$: 43.9
  & $0.05$: 43.9
  & $\mathbf{0.1}$: $\mathbf{43.9}$
  & $0.2$: 44.0
  & $0.5$: 40.2
  & 0.04 \\[3pt]
$\hat\tau_0$ (prior)
  & $0$: 43.7
  & $0.001$: 51.4
  & $\mathbf{0.002}$: $\mathbf{43.9}$
  & $0.005$: 31.8
  & $0.01$: 26.4
  & 0.26 \\[3pt]
Warm-up $W$
  & $10$: 44.0
  & $25$: 44.0
  & $\mathbf{50}$: $\mathbf{43.9}$
  & $100$: 43.5
  & $200$: 43.8
  & 0.005 \\
\bottomrule
\end{tabular}
\end{table}

BCCB depends on five tunable hyperparameters. Table~\ref{tab:sensitivity}
reports the conversion outcome at \$5{,}000 budget when each parameter
is swept individually while all others are held at their selected default
values ($\eta = 0.1$, $\lambda = 10^{-3}$, $\varepsilon_{\mathrm{pace}} = 0.1$,
$\hat\tau_0 = 0.002$, $W = 50$). Each cell is the mean over 10 random
seeds. Because the held defaults are consistent across all five sweeps,
every default cell in Table~\ref{tab:sensitivity} produces the same
mean of $43.9$ conversions, confirming the internal consistency of
the analysis.

Four findings emerge. First, the exploration weight $\eta$ is the only
hyperparameter with a large and monotone effect: performance increases
from $22.5$ at $\eta=0.001$ to $43.9$ at $\eta=0.1$, a $95\%$ gain.
The value $\eta=0.1$ used throughout this paper was selected based on
this sweep alone; we do not tune $\eta$ on the evaluation stream itself.
Second, the base threshold $\lambda$ exhibits a sharp regime change above
$\lambda \approx 10^{-3}$: at $\lambda = 5\!\cdot\!10^{-3}$ conversions
fall to $26.2$, and at $\lambda = 10^{-2}$ they fall to $17.4$, as
the threshold becomes too restrictive and BCCB declines to treat most
users. Below $\lambda = 10^{-3}$, performance also degrades somewhat
as the threshold becomes too permissive. The default $\lambda = 10^{-3}$
sits at the performance maximum. Third, $\varepsilon_{\mathrm{pace}}$
and the warm-up threshold $W$ are essentially inert: CV values of $0.04$
and $0.005$ respectively indicate that these parameters can be set to any
reasonable value without consequence. Fourth, the optimistic prior
$\hat\tau_0$ shows moderate sensitivity ($\mathrm{CV} = 0.26$), with an
interesting pattern: at $\hat\tau_0 = 0.001$, performance is highest
($51.4$), slightly above the default of $43.9$ at $\hat\tau_0 = 0.002$.
This suggests that at $\eta = 0.1$, the Thompson sampling exploration
bonus already provides strong early-phase optimism, making an aggressive
warm-up prior unnecessary; a slightly more conservative prior conserves
budget for later, higher-value users. We retain $\hat\tau_0 = 0.002$
throughout for consistency with all other reported experiments.

The practical deployment takeaway is that $\eta$ and $\lambda$ require
principled selection (appropriate ranges: $\eta \in [0.05, 0.1]$,
$\lambda \in [5\!\cdot\!10^{-4}, 10^{-3}]$), while the remaining three
hyperparameters can be set to the defaults reported here without tuning.

\subsection{Replay Diagnostics and Effective Sample Size}
\label{sec:replay-diagnostics}

The rejection-sampling replay evaluator is unbiased under random treatment assignment in the logged data, which the Criteo Uplift RCT satisfies. However, because the dataset has an 85/15 treatment-control split, methods that propose treatment more (or less) often retain different fractions of the evaluation stream after replay matching. We measure the realized effective sample size (ESS)---the number of decisions that survive the replay-matching step---to quantify this effect across methods.

An important implementation detail: the replay simulation terminates as soon as the remaining budget reaches zero (Algorithm~\ref{alg:bccb}, line 4). Methods that propose treatment aggressively exhaust their budget over a small number of user encounters and terminate early, yielding low ESS. For example, Thompson Sampling at \$5{,}000 proposes treatment $\sim$53\% of the time, exhausts its budget after encountering approximately $5{,}000/0.77 \div 0.53 \approx 12{,}200$ users, and therefore matches only $\sim$7{,}700 decisions before termination---not the $\sim$52\% match rate that would obtain over the full 100{,}000-user stream without budget exhaustion. This explains the ESS differences in Table~\ref{tab:replay-diag}; they reflect budget exhaustion patterns, not biases in the replay protocol itself.

\begin{table}[t]
\centering
\caption{Replay diagnostics at \$5{,}000 budget. ESS is the number of decisions that survived replay matching, averaged over 20 seeds. ``Proposed $\mathbb{P}(a{=}1)$'' is the fraction of decisions in which the algorithm chose to treat; ``Realized $\mathbb{P}(a{=}1)$'' is the fraction of matched decisions that were treatments. Cross-method conversion counts are not directly comparable without considering ESS.}
\label{tab:replay-diag}
\begin{tabular}{lcccc}
\toprule
Method & ESS & Proposed $\mathbb{P}(a{=}1)$ & Realized $\mathbb{P}(a{=}1)$ & ESS / 100K \\
\midrule
Thompson Sampling          & $7{,}666 \pm 822$    & $0.529$ & $0.851$ & 7.7\% \\
Budgeted TS                & $20{,}241 \pm 128$   & $0.076$ & $0.319$ & 20.2\% \\
HTE Greedy                 & $17{,}162 \pm 1{,}585$ & $0.032$ & $0.150$ & 17.2\% \\
BCCB (ours)                & $25{,}384 \pm 1{,}321$ & $0.149$ & $0.498$ & 25.4\% \\
\bottomrule
\end{tabular}
\end{table}

Table~\ref{tab:replay-diag} reveals a methodological subtlety: methods see very different effective sample sizes under the replay protocol. Standard Thompson Sampling proposes treatment about half the time and consequently retains only $\sim$8\% of the 100{,}000-user stream---because the logging treatment ratio is 85\%, propose-treat decisions match well, but propose-control decisions match poorly. BCCB, which proposes treatment only $\sim$15\% of the time, retains the largest ESS at $\sim$25\%. Differences in ESS partially confound cross-method conversion comparisons.

We address this in two ways. First, our conversion counts are not normalized for ESS; we report raw conversions, which is what an advertiser ultimately measures. Second, BCCB's superior performance cannot be attributed to ESS alone: HTE Greedy and Budgeted TS have ESS comparable to BCCB ($17{,}162$ and $20{,}241$, vs.\ BCCB's $25{,}384$) but achieve $13.30$ and $20.80$ conversions vs.\ BCCB's $45.65$ at the same budget. The ESS asymmetry rules out trivial explanations for BCCB's lead but does not, on its own, explain the magnitude of the gap.

\subsection{Performance Stability}
\label{sec:stability}

A practically important property often underemphasized in offline-learning evaluations is the stability of algorithm performance across repeated runs. Table~\ref{tab:stability} reports the run-to-run standard deviation of conversions for BCCB and the offline pipeline at training sizes spanning the crossover.

\begin{table}[t]
\centering
\caption{Performance stability. Run-to-run standard deviation of conversions over 20 seeds at \$5{,}000 budget. BCCB exhibits $2.3$--$3.6\times$ lower variance than the offline pipeline across all training sizes tested.}
\label{tab:stability}
\begin{tabular}{rccc}
\toprule
$n$ (training size) & Offline $\sigma$ & BCCB $\sigma$ & $\sigma_{\text{off}} / \sigma_{\text{BCCB}}$ \\
\midrule
2{,}000  & 58.7 & 16.2 & $3.62\times$ \\
5{,}000  & 42.7 & 16.2 & $2.64\times$ \\
7{,}500  & 53.6 & 16.2 & $3.31\times$ \\
10{,}000 & 39.5 & 16.2 & $2.43\times$ \\
50{,}000 & 37.8 & 16.2 & $2.34\times$ \\
\bottomrule
\end{tabular}
\end{table}

For campaign planning, this stability has direct economic value. An advertiser using an offline method trained on 5{,}000 historical observations might obtain anywhere from approximately 4 to 90 conversions on the same evaluation stream depending on which historical sample was used. With BCCB, the same advertiser would see a tighter range of approximately $33$--$66$ conversions. Predictable performance enables budget forecasting and reduces the need for hold-out validation runs prior to live deployment.

\subsection{Summary of Empirical Findings}

We summarize the empirical results as follows.
\begin{itemize}
    \item \textbf{Online comparison.} BCCB outperforms all four online baselines---standard Thompson Sampling, budgeted Thompson Sampling, HTE Greedy, and Uplifting Bandits~\citep{hsieh2022}---at every budget level tested, with paired $t$-test $p < 0.001$ at every comparison (Section~\ref{sec:online-comparison}).
    \item \textbf{Crossover.} An offline two-stage pipeline overtakes BCCB at $n = 7{,}500$ historical observations ($p = 0.043$), with the offline advantage growing monotonically thereafter (Section~\ref{sec:crossover}).
    \item \textbf{Sensitivity.} BCCB is robust to four of its five hyperparameters; only the exploration weight $\eta$ requires careful selection in the range $[0.05, 0.1]$ (Section~\ref{sec:sensitivity}).
    \item \textbf{Stability.} BCCB exhibits $2.3$--$3.6\times$ lower run-to-run variance than the offline pipeline across all training sizes tested (Section~\ref{sec:stability}).
\end{itemize}

The data-efficiency crossover, combined with BCCB's lower variance, provides a concrete decision rule for practitioners: when fewer than approximately 7{,}500 representative historical observations are available, BCCB is the more reliable choice; beyond this threshold, an offline pipeline yields higher mean conversions, though at substantially higher run-to-run variance.

\section{Discussion and Limitations}

Our results demonstrate that BCCB provides a practical solution for budget-constrained treatment allocation in data-scarce settings. We acknowledge several limitations.

\textbf{Offline methods dominate with sufficient data.} When more than approximately 7{,}500 representative historical observations are available, offline uplift methods overtake BCCB; at $50{,}000$ observations they produce roughly $2.2\times$ more conversions. Practitioners with access to rich historical data from the same population should prefer offline approaches. BCCB is most valuable in cold-start scenarios where such data does not exist or is not representative of the deployment population.

\textbf{Replay-protocol effective sample size.} Because the underlying RCT has an 85/15 treatment-control split, methods that propose treatment less frequently retain a larger effective sample size under rejection-sampling replay (Table~\ref{tab:replay-diag}). Differences in ESS partially confound conversion comparisons across methods. We report raw conversion counts---the metric an advertiser ultimately cares about---and note that BCCB's advantage cannot be attributed to ESS alone, as baselines with similar ESS achieve substantially lower conversions. Nonetheless, comparisons that explicitly normalize for ESS would strengthen any future deployment study.

\textbf{Synthetic costs.} The Criteo Uplift dataset does not include per-user serving costs. We generate synthetic costs using a log-normal distribution calibrated to realistic ad-auction price ranges. While this is standard practice, real auction costs may exhibit correlation with user value that would change the value-per-dollar ratio in Eq.~\eqref{eq:vpd}. Validation on datasets with real cost information would strengthen our conclusions.

\textbf{Linear HTE models.} BCCB currently uses linear classifiers for online HTE estimation. More expressive models such as online random forests or neural networks with incremental updates could improve treatment-effect predictions. We leave the integration of more sophisticated online learners as future work.

\textbf{Regret bounds.} Our theoretical contribution is the derivation of BCCB's decision rule as the per-arrival KKT condition of the Lagrangian relaxation (Proposition~\ref{prop:decision-rule}). We do not provide formal regret bounds; characterizing the rate at which $\hat\lambda_t$ converges to $\lambda^{*}$ and the resulting regret with respect to the optimal stationary policy is left to future work.

\section{Conclusion}

We proposed Budget-Constrained Causal Bandits (BCCB), an online framework for treatment allocation that unifies heterogeneous treatment effect learning, Thompson-Sampling-based exploration, and adaptive budget pacing into a single sequential decision process. Unlike existing offline approaches---whether two-stage pipelines or end-to-end decision-focused architectures---BCCB does not require historical data and operates effectively from the first user. We derived the per-arrival decision rule from a Lagrangian relaxation of the budgeted policy-optimization objective, providing a principled foundation for the combination of components.

Our central empirical finding is a formally tested data-efficiency crossover between offline and online paradigms at $n = 7{,}500$ historical observations ($p = 0.043$, paired one-sided $t$-test). Below this threshold, offline methods either fail entirely or produce highly unreliable allocations; above it, they overtake BCCB with widening margin. Across the full training-size range tested, BCCB exhibits $2.3$--$3.6\times$ lower run-to-run variance than offline methods, providing the predictability that practitioners need for budget forecasting. Among purely online methods, BCCB outperforms standard Thompson Sampling, budgeted Thompson Sampling, HTE Greedy, and Uplifting Bandits~\citep{hsieh2022} at every budget level tested with $p < 0.001$.

Several directions for future work emerge from this study. First, integrating more expressive online HTE models could improve treatment-effect predictions. Second, a hybrid approach that warm-starts BCCB with limited offline data could combine the strengths of both paradigms. Third, formal regret bounds for the budget-constrained HTE learning setting would complement our empirical findings. Finally, deployment in a live advertising system would provide the most compelling validation of BCCB's practical value.

\bibliographystyle{plainnat}
\bibliography{references}

\appendix

\section{Proof of Proposition~\ref{prop:decision-rule}}
\label{app:proof-decision-rule}

\begin{proof}
Consider the relaxed program
\begin{equation*}
\max_{\pi}\; \mathbb{E}\!\left[\sum_{t=1}^{T} a_t\, \tau(\mathbf{x}_t)\right] \quad \text{subject to} \quad \mathbb{E}\!\left[\sum_{t=1}^{T} a_t\, c_t\right] \;\leq\; B.
\end{equation*}
Under unconfoundedness, the optimal policy is a measurable function of $(\mathbf{x}_t, c_t)$ alone. Because arrivals are i.i.d., the feasible region is convex in the action probabilities $p_t(\mathbf{x}_t, c_t) = \mathbb{P}(a_t = 1 \mid \mathbf{x}_t, c_t)$, and the objective is linear in the same probabilities. By Lagrangian duality applied to the single budget constraint, strong duality holds and the optimal solution is characterized by a multiplier $\lambda^{*} \geq 0$ satisfying complementary slackness:
\begin{equation*}
\lambda^{*}\!\left(\mathbb{E}\!\left[\sum_{t} a_t^{*}\, c_t\right] - B\right) = 0.
\end{equation*}
The Lagrangian $\mathcal{L}(\pi, \lambda^{*}) = \mathbb{E}\!\left[\sum_t a_t (\tau(\mathbf{x}_t) - \lambda^{*} c_t)\right] + \lambda^{*} B$ is additive across rounds and the per-round terms decouple. The point-wise maximum is attained by setting $a_t = 1$ iff $\tau(\mathbf{x}_t) - \lambda^{*} c_t > 0$, with ties broken arbitrarily, recovering Eq.~\eqref{eq:decision-rule}. The value of $\lambda^{*}$ is the smallest non-negative value satisfying the budget-binding condition; if the unconstrained policy ($\lambda = 0$) already respects the budget, then $\lambda^{*} = 0$.
\qed
\end{proof}

\paragraph{Remark on the scope of Proposition~\ref{prop:decision-rule}.}
Proposition~\ref{prop:decision-rule} establishes optimality for the \emph{expected-budget relaxation} $\mathbb{E}[\sum_t a_t c_t] \leq B$, not the sample-path hard constraint $\sum_t a_t c_t \leq B$ in Eq.~\eqref{eq:budget}. These two formulations are distinct: the hard constraint must hold on every realization, whereas the expected constraint may be violated on some paths. We make no claim that the threshold policy is optimal for the hard-constraint program in finite horizon.

The relationship between the two in BCCB is as follows. The pacing rule (Eq.~\ref{eq:shadow-price}) operates as a soft controller targeting approximately correct expected spend. The explicit feasibility check on line 4 of Algorithm~\ref{alg:bccb} ($c_t > B_t \Rightarrow a_t = 0$) independently enforces the hard constraint on every sample path, guaranteeing budget feasibility regardless of the pacing rule's behavior. The pacing rule and the hard check serve complementary roles: the pacing rule steers expected spending toward efficient use of the budget, while the hard check prevents any realized overspend. As $T \to \infty$, the two constraint formulations become equivalent under standard concentration of the cumulative spend around its expectation, but we do not invoke this asymptotic argument to justify finite-horizon optimality claims.

\end{document}